\documentclass[a4paper,fleqn]{cas-sc}

\usepackage[numbers,sort&compress]{natbib}
\usepackage{amsthm}
\usepackage{algorithm}
\usepackage{algpseudocode}

\newtheorem{theorem}{Theorem}[section]

\newtheorem{definition}[theorem]{Definition}

\theoremstyle{remark}

\def\tsc#1{\csdef{#1}{\textsc{\lowercase{#1}}\xspace}}
\tsc{WGM}
\tsc{QE}
\tsc{EP}
\tsc{PMS}
\tsc{BEC}
\tsc{DE}

\begin{document}
\let\WriteBookmarks\relax
\def\floatpagepagefraction{1}
\def\textpagefraction{.001}
\shorttitle{A Residual-Shell-Based Lower Bound for ORC}
\shortauthors{X. Gu, et al.}

\title [mode = title]{A Residual-Shell-Based Lower Bound for Ollivier–Ricci Curvature}                      



\author[1]{Xiang Gu}
\ead{xianggu@xjtu.edu.cn}

\credit{Conceptualization of this study, Methodology, Software}

\affiliation[1]{organization={School of Mathematics and Statistics, Xi'an Jiaotong University},
                city={Xi'an},
                postcode={710049}, 
                state={Shaanxi},
                country={China}}

\author[2]{Huichun Zhang}
\ead{zhanghc3@mail.sysu.edu.cn}

\author[1]{Jian Sun}
\ead{jiansun@xjtu.edu.cn}

\affiliation[2]{organization={School of Mathematics, Sun Yat-sen University},
                city={Guangzhou},
                postcode={510275},
                state={Guangdong},
                country={China}}
\cormark[1]
\credit{Data curation, Writing - Original draft preparation}
\cortext[cor1]{Corresponding author}





\begin{abstract}
As a powerful and descriptive invariant, curvature plays a fundamental role in graph analysis. In particular, Ollivier--Ricci curvature (ORC), defined via the Wasserstein distance that captures rich geometric information, has received growing attention in both theories and applications. However, the high computational cost of Wasserstein distance evaluation has significantly limited the broader practical use of ORC. To alleviate this issue, \citet{kang2024accelerated} proposed a computationally efficient lower bound as a proxy for ORC defined based on 1-hop random walks, but this approach empirically suffers from large gaps against the exact ORC. In this paper, we establish a substantially tighter lower bound based on residual shell for ORC than that of \citet{kang2024accelerated}, while retaining much lower computational cost than exact ORC computation with speedups of tens of times in practice. Moreover, our bound is not restricted to 1-hop random walks, but is applicable for $k$-hop random walks ($k>1$). Experiments on several fundamental graph structures demonstrate the effectiveness of our bound in terms of both approximation accuracy and computational efficiency. The code is available at \url{https://github.com/XJTU-XGU/ORC_bound}.
\end{abstract}




\maketitle

\section{Introduction}
Curvature, a fundamental and illuminating invariant reflecting the local shape feature of a space or graph, plays an important role in both mathematical theories and real-world applications. In particular, the Ollivier-Ricci curvature (ORC)~\citep{ollivier2007ricci,ollivier2009ricci} represents a key curvature measure for graphs, offering insights into variations observed in random walks using the Wasserstein distance from optimal transport. Based on optimal transport and probability theory, ORC integrates geometric insights into the analysis of graphic structures, and has been demonstrated to be useful in various scenarios, such as in assessing disparities between real-world networks \citep{samal2018comparative} and in finding bottlenecks in real-world networks \citep{gosztolai2021unfolding}.

While impressive, ORC suffers from high computational costs due to the evaluation of the Wasserstein distance, which limits its practical applicability, especially for large-scale graphs. To address this challenge, a natural strategy is to develop a computationally efficient proxy for the Wasserstein distance. In this context, a tight lower bound of ORC is particularly valuable, as it provides a theoretically guaranteed and computationally efficient surrogate for the true curvature in large-scale graph analysis.
Specifically, \citet{jost2014ollivier} presented a lower bound of ORC in continuous metric spaces, which was recently extended by \citet{kang2024accelerated} to a broader context, including spaces with metrics on integers with elements of laziness incorporated into their structure, allowing applicability to a wider variety of computational scenarios. 

Although the lower bound proposed in \citep{jost2014ollivier,kang2024accelerated} can be computed efficiently, it suffers from two major limitations. First, the gap between this bound and the true ORC can be considerable, resulting in substantial approximation error. Second, it is tailored only to 1-hop random walks, thereby restricting its applicability to ORC defined through \(k\)-hop random walks for \(k>1\). 

To address these issues, in this paper, we derive a new lower bound for ORC based on a shell-wise coupling argument. Our main idea is to decompose the transport across graph-distance shells, match as much mass as possible within low-cost local shells, and bound the remaining residual transport in a controlled way. This yields an explicit upper bound for the Wasserstein distance and hence a new residual-shell (RS) based lower bound for ORC. Our bound inherently applies to general ORC defined based on \(k\)-hop random walks. Meanwhile, we present an algorithm to compute the derived lower bound. 
Experimentally, on several kinds of basic graph structures, we show that our RS-based lower bound is much tighter than that in \citep{kang2024accelerated}, and retains much smaller computational cost than direct ORC computation, achieving speedups of tens of times in practice.

\section{Preliminaries}

\paragraph{Graph theory.}
A graph \(G=(V,E)\) consists of a set of vertices \(V\) and a set of edges \(E\), where each edge connects a pair of vertices. This paper focuses on undirected graphs, in which edges are unordered pairs. Two vertices \(u,v\in V\) are said to be adjacent, denoted by \(u\sim v\), if there is an edge \(e=\{u,v\}\in E\) connecting them. We consider locally finite graphs, meaning that each vertex is incident to only finitely many edges. A simple graph contains at most one edge between any pair of vertices and has no self-loops. {Two vertices are connected if there exists a path between them.
Following~\cite{kang2024accelerated}, in this paper, we assign every edge a length of one, and the graph distance \(d(u,v)\) is defined as the length of a shortest path from \(u\) to \(v\). In a weighted graph, each edge is assigned a positive numerical weight, and the weight between two adjacent vertices \(u\) and \(v\) is denoted by \(w_{uv}\), with \(w_{uv}=w_{vu}\).} 
The $k$-hop neighborhood of a vertex \(v\), denoted by \(\mathcal N_k(v)\), is defined by $\mathcal N_k(v) = \{v':d(v',v)\leqslant k\}$.
For simplicity, we denote $\mathcal N_1(v)$ as $\mathcal N(v)$. 
We define $k$-hop degree of a vertex \(v\) as $d_v^k=|\mathcal N_k(v)|$. If $k=1$, $d_v^1=|\mathcal N(v)|$ is the degree in other graph literature.  

\paragraph{Ollivier-Ricci curvature.} 
{For any node \(x\in V\), let \(P\) denote the transition matrix of the weighted random walk on \(G\), i.e.,
$P(u,v)= {w_{uv}}/{\sum_{v'\in\mathcal N(u)}w_{uv'}}, \text{ if } \{u,v\}\in E$, otherwise 0. Note that for unweighted graphs, $w_{uv}=1$}. We define the \(k\)-hop lazy random walk measure centered at \(x\) as
\begin{equation}\label{eq:k_hop_measure}
 \mu_x^k(y)=(1-\alpha)\tilde{\mu}_x^k(y)+\alpha\delta_x(y),
 \text{ with }
\tilde{\mu}_x^k(y)=P^k(x,y),   
\end{equation}
where \(\alpha\in(0,1)\), $P^k$ is $k$-th power of matrix $P$, and \(\delta_x\) is the Dirac measure at \(x\).
Given any two local measures $\mu_x^k$ and $\mu_y^k$, the Wasserstein distance between them is defined by 
\begin{equation}
    W_1(\mu_x^k,\mu_y^k) = \min_{\gamma\in \Pi(\mu_x^k,\mu_y^k)} \sum_{x'\in \mathcal N_k(x)} \sum_{y'\in \mathcal N_k(y)} d(x',y')\gamma(x',y'),
\end{equation}
where $\Pi(\mu_x^k,\mu_y^k)$ is the set of joint distributions with marginals $\mu_x^k$ and $\mu_y^k$, i.e., 
\begin{equation}
    \Pi(\mu_x^k,\mu_y^k)= \left\{ \gamma\in\mathbb R^{d_x^k\times d_y^k}:\gamma\geqslant0, \sum_{y'\in\mathcal N_k(y)}\gamma(x',y')=\mu_x^k(x'),  \sum_{x'\in\mathcal N_k(x)}\gamma(x',y')=\mu_y^k(y')\right\}.
\end{equation}
Intuitively, the Wasserstein distance is the minimum cost of transport the mass from $\mu_x^k$ to $\mu_y^k$.  
We now introduce the definition of Ollivier-Ricci curvature on graphs. 
\begin{definition}[Ollivier-Ricci curvature]
    Given the graph $(V,E)$, for any nodes $x,y\in V$, the Ollivier-Ricci curvature based on $k$-hop random walks is defined by 
    \begin{equation}
        \kappa(x,y) = 1 - \frac{W_1(\mu^k_x,\mu^k_y)}{d(x,y)},
    \end{equation}
    where $\mu^k_x,\mu^k_y$ are the $k$-hop measures defined in \eqref{eq:k_hop_measure}.
\end{definition}
Due to the high computational cost of the Wasserstein distance, computing ORC remains challenging. To alleviate this issue, \citet{jost2014ollivier} derived a lower bound for ORC that can be computed more efficiently, which was extended by \citet{kang2024accelerated}. However, those bounds are limited to the ORC defined based on $1$-hop random walks. Meanwhile, our empirical results indicate that the bound in \citep{kang2024accelerated} incurs substantial approximation error when used to estimate ORC. This motivates us to develop a tighter bound for ORC under general $k$-hop random walks.

\section{Bound for Ollivier-Ricci Curvature}
Our main idea for deducing the lower bound is to decompose the transport across graph-distance shells, and then match as much mass as possible within low-cost local shells. Finally, we bound the remaining residual transport in a controlled way. For $k$-hop local measures $\mu_x^k,\mu_y^k$, we define the shell sets as 
\begin{equation}
    E_r = \{(u,v)\in \mathcal{N}_k(x)\times \mathcal{N}_k(y): d(u,v)=r\}, 
\; r=0,1,\ldots,l.
\end{equation} We construct a partial transport plan shell by shell as follows. 
We initialize the residual marginals by
\begin{equation}
a_u^{(0)}=\mu_x^k(u),
\;
b_v^{(0)}=\mu_y^k(v),
\; \mbox{ for } u\in \mathcal{N}_k(x),\ v\in \mathcal{N}_k(y).
\end{equation}
For each shell \(r=0,1,\ldots,l\), let $f_{uv}^{(r)} \mbox{ for } (u,v)\in E_r$ be any nonnegative allocation satisfying $\sum_{v:(u,v)\in E_r} f_{uv}^{(r)} \leqslant a_u^{(r)},
\sum_{u:(u,v)\in E_r} f_{uv}^{(r)} \leqslant b_v^{(r)}.$
Then we update the residual marginals by
\begin{equation}
  a_u^{(r+1)}
=
a_u^{(r)}-\sum_{v:(u,v)\in E_r} f_{uv}^{(r)},
\quad
b_v^{(r+1)}
=
b_v^{(r)}-\sum_{u:(u,v)\in E_r} f_{uv}^{(r)}.
\end{equation}
We define the transported mass $m_r$ on the $r$-th shell and residual mass $R_l$ after the $l$-th shell by
\begin{equation}
m_r = \sum_{(u,v)\in E_r} f_{uv}^{(r)}, \quad R_l = 1-\sum_{r=0}^l m_r.
\end{equation}
Typically, $f_{uv}^{(r)}$ is set as $f_{uv}^{(r)} = \min\{a_u^{(r)},\, b_v^{(r)}\}$. With above constructing process, we have the following theorem. 


\begin{theorem}[Residual-shell-based bound]
\label{thm:residual_shell_upper_bound}
Under the above notations, let $\bar r$ be any constant such that
$
d(u,v)\leqslant \bar r \text{ whenever } a_u^{(l+1)}>0 \text{ and } b_v^{(l+1)}>0.
$
Then, the Wasserstein distance satisfies
\begin{equation}
   W_1(\mu_x^k,\mu_y^k) \leqslant \sum_{r=0}^l r\,m_r + \bar r\,R_l. 
\end{equation}
Consequently, we have
\begin{equation}
    \kappa(x,y)\geqslant 1-\frac{\sum_{r=0}^l r\,m_r + \bar r\,R_l}{d(x,y)}.
\end{equation}
\end{theorem}

\begin{proof}
We construct a feasible coupling $\pi\in\Pi(\mu_x^k,\mu_y^k)$ yielding the bound in two parts.
First, we aggregate the mass transported within the first $l$ shells to form a partial plan $\pi_{\leqslant l}$.
Then we redistribute the remaining mass via an auxiliary coupling whose transport cost is uniformly bounded by $\bar r$.

\textit{Step 1: Construction of the partial transport plan.}
Define a partial transport plan on $\mathcal{N}_k(x)\times \mathcal{N}_k(y)$ by
\begin{equation}
   \pi_{\leqslant l}(u,v) = \sum_{r=0}^l f_{uv}^{(r)},
\end{equation}
with the convention that $f_{uv}^{(r)}=0$ whenever $(u,v)\notin E_r$.
By construction, we have $\pi_{\leqslant l}(u,v)\geqslant 0$, and its support is contained in the union of the first $l$ shells, namely,
$\operatorname{supp}(\pi_{\leqslant l}) \subseteq \bigcup_{r=0}^l E_r.$

\textit{Step 2: Residual marginals and remaining mass.}
Next, define the residual marginals
\begin{equation}\label{eq:au_bv}
    a(u) = \mu_x^k(u)-\sum_{v\in \mathcal{N}_k(y)}\pi_{\leqslant l}(u,v),
\; 
    b(v) = \mu_y^k(v)-\sum_{u\in \mathcal{N}_k(x)}\pi_{\leqslant l}(u,v).
\end{equation}
By the update rule, we have $ a(u)=a_u^{(l+1)}\geqslant 0, b(v)=b_v^{(l+1)}\geqslant 0.$
Moreover, summing over $u$ and $v$ yields
\begin{equation}
\sum_{u\in \mathcal{N}_k(x)} a(u)
=
\sum_{v\in \mathcal{N}_k(y)} b(v)
= \sum_{u\in \mathcal{N}_k(x)} \mu_x^k(u)-\sum_{u\in \mathcal{N}_k(x)}\sum_{v\in \mathcal{N}_k(y)} \sum_{r=0}^l f_{uv}^{(r)}
=
1-\sum_{r=0}^l\sum_{(u,v)\in E_r}f_{uv}^{(r)}
=
R_l.
\end{equation}

\textit{Step 3: Completion of the coupling via residual transport.}
We now complete the residual mass into a full coupling.
If $R_l=0$, we set $\pi_{>l}\equiv 0$.
Otherwise, define the normalized residual measures
\begin{equation}
   \hat a(u) = \frac{a(u)}{R_l},
\; 
   \hat b(v) = \frac{b(v)}{R_l}.
\end{equation}
Then $\hat a$ and $\hat b$ are both probability measures on $\mathcal{N}_k(x)$ and $\mathcal{N}_k(y)$, respectively.
Let $\gamma\in\Pi(\hat a,\hat b)$ be any coupling between them.
By the definition of $\bar r$, every pair $(u,v)$ with $a(u)>0$ and $b(v)>0$ satisfies $d(u,v)\leqslant \bar r$, and hence
\begin{equation}
    \sum_{u,v} d(u,v)\,\gamma(u,v)\leqslant \bar r.
\end{equation}
Define
$\pi_{>l}=R_l\,\gamma.$
Then $\pi_{>l}$ has marginals $a$ and $b$. Defining $\pi=\pi_{\leqslant l}+\pi_{>l}$
and using \eqref{eq:au_bv}, we have $\pi\in\Pi(\mu_x^k,\mu_y^k)$.

\textit{Step 4: Transport cost decomposition.}
Since $d(u,v)=r$ on $E_r$, we have
\begin{equation}
    \sum_{u,v} d(u,v)\,\pi_{\leqslant l}(u,v)
=\sum_{r=0}^l \sum_{(u,v)\in E_r} r\,f_{uv}^{(r)}
=\sum_{r=0}^l r\,m_r.
\end{equation}
For the residual part,
\begin{equation}
    \sum_{u,v} d(u,v)\,\pi_{>l}(u,v)
=
R_l \sum_{u,v} d(u,v)\,\gamma(u,v)
\leqslant \bar r\,R_l.
\end{equation}

\textit{Step 5: Conclusion via Wasserstein bound.}
Combining the above two displays gives
\begin{equation}
\sum_{u,v} d(u,v)\,\pi(u,v)
\leqslant
\sum_{r=0}^l r\,m_r+\bar r\,R_l.
\end{equation}
Since $\pi$ is a feasible coupling of $\mu_x^k$ and $\mu_y^k$, by the definition of $W_1$, we have
\begin{equation}
W_1(\mu_x^k,\mu_y^k)
\leqslant
\sum_{u,v} d(u,v)\,\pi(u,v)
\leqslant
\sum_{r=0}^l r\,m_r+\bar r\,R_l.
\end{equation}

Finally, using the definition of ORC, we have
\begin{equation}
    \kappa(x,y)
\geqslant
1-\frac{\sum_{r=0}^l r\,m_r+\bar r\,R_l}{d(x,y)}.
\end{equation}
\end{proof}

For fixed $\bar r$, the upper bound on the Wasserstein distance is monotone non-increasing with respect to $l$. Therefore, a larger $l$ may lead to a tighter bound. In practice, we set 
\begin{equation}
   \bar r = \max\{d(u,v):a(u)>0,b(v)>0\}. 
\end{equation}
We emphasize that our bound applies to general $k$-hop lazy random walk measures, while the standard random walk and the 1-hop measures are both included as special cases. 
{Unlike some combinatorial formulations that characterize Ollivier–Ricci curvature using local cycle counts (e.g., \citep{kelly2019self}), our residual-shell construction relies only on shortest-path distances and is applicable to graphs with cycles of different lengths.}
We provide the algorithm \ref{alg:residual_shell_bound} to compute the bound.

\begin{algorithm}[t]
\caption{Residual-shell-based bound for $\kappa(x,y)$}
\label{alg:residual_shell_bound}
\begin{algorithmic}[1]
\Require Nodes $x,y$, measures $\mu_x^k,\mu_y^k$, shell depth $l$, shortest path length $d$
\Ensure Lower bound ${\kappa}_{xy}$

\State Initialize residual masses: $a_u \gets \mu_x^k(u)$ and $b_v \gets \mu_y^k(v)$ for $u\in\mathcal N_k(x)$ and $v\in\mathcal N_k(y)$
\State Set $m_r \gets 0$ for all $r=0,\dots,l$

\For{$r=0,1,\dots,l$}
    \For{each $(u,v)\in E_r$}
        \State $\delta \gets \min\{a_u,b_v\}$
        \State $f_{uv}^{(r)} \gets \delta$
        \State $a_u \gets a_u-\delta$, $b_v \gets b_v-\delta$, and $m_r \gets m_r+\delta$
    \EndFor
\EndFor

\State Define $R_l \gets \sum_{u\in\mathcal N_k(x)} a_u$ and $\bar r \gets \max\{d(u,v):a(u)>0,b(v)>0\}$
\State Compute $U_l \gets \sum_{r=0}^l r\,m_r+\bar r\,R_l$ and ${\kappa}_{xy}\gets 1-\frac{U_l}{d(x,y)}$
\State \Return ${\kappa}_{xy}$
\end{algorithmic}
\end{algorithm}

\section{Numerical Evaluation}
This section evaluates the computational efficiency and tightness of our residual-shell-based bound for ORC. We consider three approaches for computing the ORC or its proxy: the exact OCR that computes the Wasserstein distance by linear programming, the lower bounds proposed by \citet{kang2024accelerated} (KP-LB), and our proposed residual-shell-based lower bound (RS-LB). {The evaluation involves unweighted synthetic graphs, unweighted and weighted real-world graphs. For synthetic graphs, we generate representative graphs, including Erdős--Rényi (ER) graphs, Barabási--Albert (BA) graphs, Watts--Strogatz (WS) graphs, and planar regular grids. The planar regular grid includes triangular, square, and hexagonal grids.  We use the Python package \texttt{networkx} to generate the synthetic graphs. For larger real-world graphs, we employ two widely adopted real-world graph datasets, Amazon \cite{dou2020enhancing} and Facebook \cite{xu2022contrastive}, in graph anomaly detection. Amazon (Am) contains 10224 nodes and 175608 edges; Facebook (FB) includes 1081 nodes and }{27552 edges. These two graphs are originally unweighted. Additionally, we build a weighted graph based on Amazon and Facebook (denoted as WAm and WFB, respectively) by defining the weight of each edge $(u,v)$ as $w_{uv}=\exp(-\frac{1}{2\sigma^2}{\|f_u-f_v\|_2^2}).$ Here $f_u$ is the feature of node $u$, and $\sigma^2$ is the variance of all feature distances. }

\begin{figure}
    \centering
    \includegraphics[width=1\linewidth]{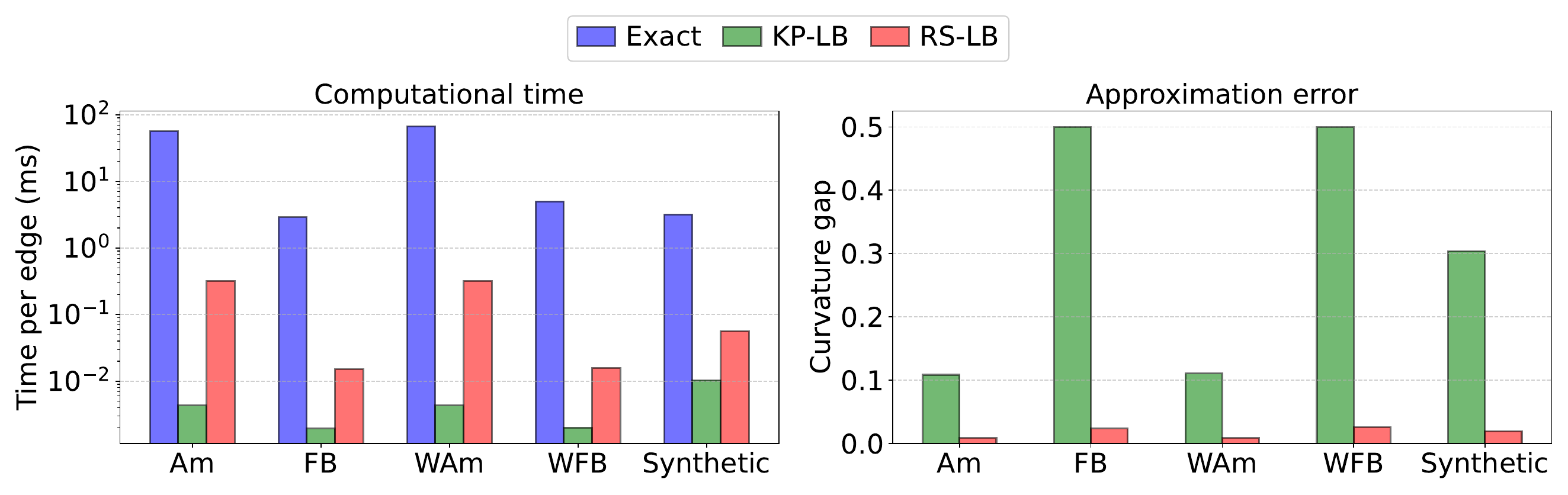}
    \caption{Left: computational time cost of exact ORC computed using linear programming, the lower bounds including KP-LB~\citep{kang2024accelerated} and our RS-LB, measured via runtime per edge. Right: approximation error of KP-LB and RS-LB, measured by the mean absolute gap between the bound and the exact ORC. Note that the results on synthetic graphs are aggregated.}
    \label{fig:time_error}
\end{figure}

\begin{figure}
    \centering
    \includegraphics[width=1\linewidth]{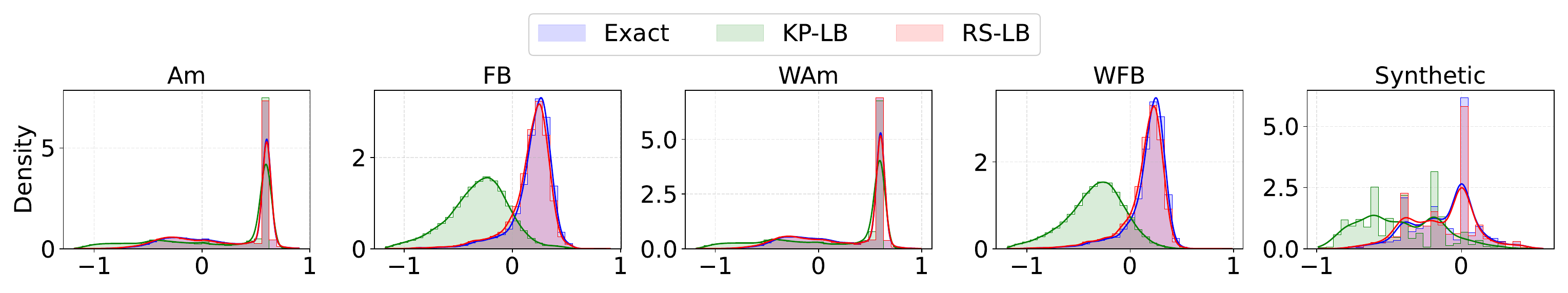}
    \caption{Histograms of the exact curvature and its lower bounds including KP-LB~\citep{kang2024accelerated} and our derived RS-LB, on four graphs. The curves are the corresponding estimated density functions.}
    \label{fig:hist}
\end{figure}
\begin{figure}
    \centering
    \includegraphics[width=0.98\linewidth]{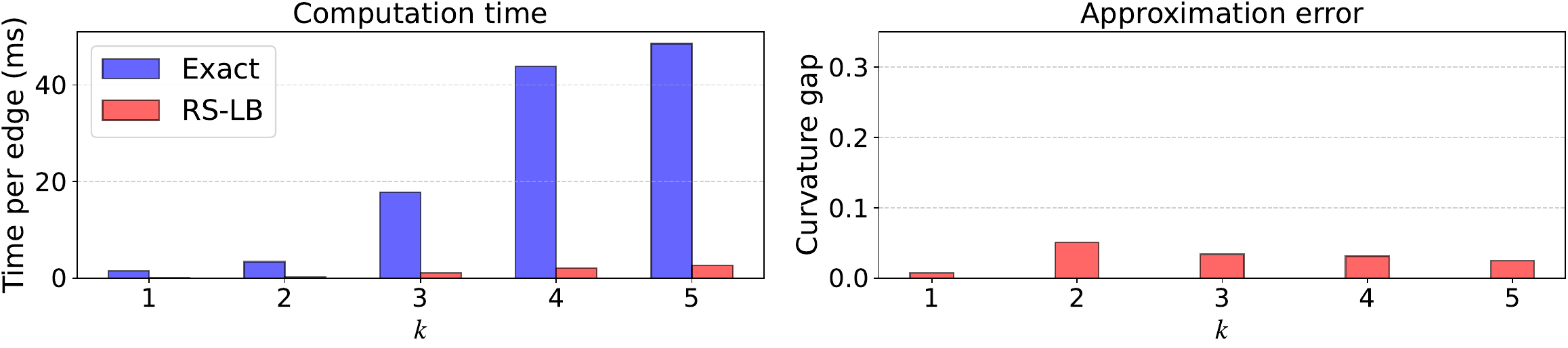}
    \caption{Computational time and approximation error of RS-LB for ORC defined on $k$-hop random walks for varying $k$.}
    \label{fig:k_hop}
\end{figure}

In the implementation of algorithm~\ref{alg:residual_shell_bound}, we set $l=3$, as larger values yield similar bounds. To compare with KP-LB~\cite{kang2024accelerated}, we first conduct experiments on 1-hop lazy random walks with $\alpha=0.4$. Figure~\ref{fig:time_error} compares the computational time of three methods, and the approximation error of KP-LB and our derived RS-LB, measuring the tightness of bounds. We can observe that our bound RS-LB takes more time than KP-LB, but still yields speedups of tens of times over the exact computation of ORC. In the right part of Figure~\ref{fig:time_error}, our RS-LB achieves a significantly smaller mean absolute gap, indicating that RS-LB is much tighter than KP-LB. In Figure~\ref{fig:hist}, we plot the histograms of the exact ORC and values of KP-LB and RS-LB. We can see that the histograms of our RS-LB are much closer to those of the exact ORC, demonstrating a better approximation achieved by RS-LB. In particular, on the Grid graph, RS-LB perfectly matches the exact ORC and takes much shorter computational time. 

We also evaluate our RS-LB on ORC defined on $k$-hop random walks.  Note that KP-LB is not applicable to this $k$-hop setting. 
Figure~\ref{fig:k_hop} reports the computational time and approximation error of RS-LB for ORC defined on $k$-hop random walks.  As demonstrated in Figure~\ref{fig:k_hop}, RS-LB yields smaller approximation errors (below 0.05) and achieves 10- to 20-fold speedups across different $k$.

\section{Concluding Remarks}
This paper established a residual-shell-based lower bound (RS-LB) for the Ollivier--Ricci curvature. Numerical evaluations demonstrated that RS-LB is tighter than existing bounds while accelerating the computation by a factor of tens. Future work will focus on exploring broader downstream applications of RS-LB.
\vspace{0.2cm}


\noindent\textbf{Acknowledgment.} This work was supported by the National Key R\&D Program 2021YFA1003000, NSFC (12125104, 12501709).

\noindent\textbf{Declaration of generative AI and AI-assisted technologies in the manuscript preparation process.}
During the preparation of this work the authors used ChatGPT in order to polish the writing. After using this tool/service, the authors reviewed and edited the content as needed and take full responsibility for the content of the published article.




\bibliographystyle{cas-model2-names}

\bibliography{cas-refs}






\end{document}